\documentclass{article}

\usepackage{times}
\usepackage{graphicx} 
\usepackage{subcaption}

\usepackage{natbib}

\usepackage{algorithm}
\usepackage{algorithmic}

\usepackage{multirow}
\usepackage{rotating}

\usepackage{hyperref}

\usepackage[accepted]{icml2016}

\usepackage{amsthm,amsmath,amssymb}
\usepackage{amssymb}
\usepackage{stackrel}
\usepackage[draft]{fixme}
\fxsetup{inline,nomargin,theme=color}
\FXRegisterAuthor{ufx}{uanfx}{\color{blue}Uri}

\usepackage{tikz}
\usetikzlibrary{arrows}

\def\cN{\mathcal{N}}

\def\cH{\mathcal{H}}

\def\diag{\mbox{diag}}

\def\hypspace{\cH}
\def\repspace{\cN}
\def\lridge{{\sc Lasso + Ridge}}

\DeclareMathOperator*{\argmin}{arg\,min}

\DeclareMathOperator*{\expect}{\mathbb{E}}
\DeclareMathOperator*{\disc}{disc}

\newcommand{\R}{\mathbb{R}}

\theoremstyle{definition}
\newtheorem{definition}{Definition}
\theoremstyle{lemma}
\newtheorem{lemma}{Lemma}
\theoremstyle{theorem}
\newtheorem{theorem}{Theorem}
\theoremstyle{assumption}

\theoremstyle{proposition}

\newtheorem{apptheorem}{Theorem}
\newtheorem{innercustomthm}{Theorem}
\newenvironment{customthm}[1]
  {\renewcommand\theinnercustomthm{#1}\innercustomthm}
  {\endinnercustomthm}

\newcommand{\pF}{\hat{P}_\Phi^{F}}
\newcommand{\pC}{\hat{P}_\Phi^{CF}}
\newcommand{\discPQ}{\text{disc}(\pF,\pC)}
\newcommand{\discPQH}{\text{disc}_\hypspace(\pF,\pC) }
\newcommand{\discPQHl}{\text{disc}_{\hypspace_l}(\pF,\pC) }
\newcommand{\bfp}{\hat{\beta}^F(\Phi)}
\newcommand{\bcp}{\hat{\beta}^{CF}(\Phi)}

\newcommand{\mytitle}{Learning Representations for Counterfactual Inference}

\icmltitlerunning{\mytitle}

\begin{document}

\twocolumn[
\icmltitle{\mytitle}

\icmlauthor{Fredrik D. Johansson$^*$}{frejohk@chalmers.se}
\icmladdress{CSE, Chalmers University of Technology, G\"{o}teborg, SE-412 96, Sweden}
\icmlauthor{Uri Shalit$^*$}{shalit@cs.nyu.edu}
\icmlauthor{David Sontag}{dsontag@cs.nyu.edu}
\icmladdress{CIMS, New York University, 251 Mercer Street, New York, NY 10012 USA}
\icmlauthor{$^*$ Equal contribution}{}

\icmlkeywords{causal inference, representation learning, domain adaptation}

\vskip 0.3in
]

\begin{abstract}
Observational studies are rising in importance due to the widespread accumulation of data in fields such as healthcare, education, employment and ecology. We consider the task of answering counterfactual questions such as, ``Would this patient have lower blood sugar had she received a different medication?''.
We propose a new algorithmic framework for counterfactual inference which brings together ideas from domain adaptation and representation learning. In addition to a theoretical justification, we perform an empirical comparison with previous approaches to causal inference from observational data. Our deep learning algorithm significantly outperforms the previous state-of-the-art.

\end{abstract}

\section{Introduction}

Inferring causal relations is a fundamental problem in the sciences and commercial applications. The problem of causal inference is often framed in terms of \emph{counterfactual} questions \citep{lewis1973causation,rubin1974estimating,pearl2009causality} such as ``Would this patient have lower blood sugar had she received a different medication?'', or ``Would the user have clicked on this ad had it been in a different color?''.
In this paper we propose a method to learn representations suited for counterfactual inference, and show its efficacy in both simulated and real world tasks.

We focus on counterfactual questions raised by what are known as \emph{observational studies}. Observational studies are studies where interventions and outcomes have been recorded, along with appropriate context. For example, consider an electronic health record dataset collected over several years, where for each patient we have lab tests and past diagnoses, as well as data relating to their diabetic status, and the causal question of interest is which of two existing anti-diabetic medications A or B is better for a given patient. Observational studies are rising in importance due to the widespread accumulation of data in fields such as healthcare, education, employment and ecology. We believe machine learning will be called on more and more to help make better decisions in these fields, and that researchers should be careful to pay attention to the ways in which these studies differ from classic supervised learning, as explained in Section \ref{Sec:prob} below.

In this work we draw a connection between counterfactual inference and domain adaptation. We then introduce a form of regularization by enforcing similarity between the distributions of representations learned for populations with different interventions. For example, the representations for patients who received medication A versus those who received medication B. This reduces the variance from fitting a model on one distribution and applying it to another. In Section \ref{Sec:model} we give several methods for learning such representations. In Section \ref{Sec:theory} we show our methods approximately minimizes an upper bound on a regret term in the counterfactual regime. The general method is outlined in Figure \ref{fig:model}.
Our work has commonalities with recent work on learning fair representations \citep{zemel2013learning,louizos2015variational} and learning representations for transfer learning \citep{ben2007analysis,gani2015domain}. In all these cases the learned representation has some invariance to specific aspects of the data: either an identity of a certain group such as racial minorities for fair representations, or the identity of the data source for domain adaptation, or, in the case of counterfactual learning, the type of intervention enacted in each population.

In machine learning, counterfactual questions typically arise in problems where there is a learning agent which performs actions, and receives feedback or reward for that choice without knowing what would be the feedback for other possible choices. This is sometimes referred to as bandit feedback \citep{beygelzimer2010contextual}. This setup comes up in diverse areas, for example off-policy evaluation in reinforcement learning \citep{sutton1998reinforcement}, learning from ``logged implicit exploration data'' \citep{strehl2010learning} or ``logged bandit feedback'' \citep{swaminathan2015batch}, and in understanding and designing complex real world ad-placement systems \cite{bottou2013counterfactual}.
Note that while in contextual bandit or robotics applications the researcher typically knows the method underlying the action choice (e.g. the policy in reinforcement learning), in observational studies we usually do not have control or even a full understanding of the mechanism which chooses which actions are performed and which feedback or reward is revealed. For instance, for anti-diabetic medication, more affluent patients might be insensitive to the price of a drug, while less affluent patients could bring this into account in their choice.

Given that we do not know beforehand the particulars determining the choice of action, the question remains, how can we learn from data which course of action would have better outcomes. By bringing together ideas from representation learning and domain adaptation, our method offers a novel way to leverage increasing computation power and the rise of large datasets to tackle consequential questions of causal inference.

The contributions of our paper are as follows. First, we show how to formulate the problem of counterfactual inference as a domain adaptation problem, and more specifically a covariate shift problem. Second, we derive new families of representation algorithms for counterfactual inference: one is based on linear models and variable selection, and the other is based on deep learning of representations \citep{bengio2013representation}. Finally, we show that learning representations that encourage similarity (balance) between the treated and control populations leads to better counterfactual inference; this is in contrast to many methods which attempt to create balance by re-weighting samples \citep[e.g.,][]{bang2005doubly,dudik2011doubly,austin2011introduction,swaminathan2015batch}. We show the merit of learning balanced representations both theoretically in Theorem 1, and empirically in a set of experiments across two datasets.

\section{Problem setup}
\label{Sec:prob}

Let $\mathcal{T}$ be the set of potential interventions or actions we wish to consider, $\mathcal{X}$ the set of contexts, and $\mathcal{Y}$ the set of possible outcomes. For example, for a patient $x \in \mathcal{X}$ the set $\mathcal{T}$ of interventions of interest might be two different treatments, and the set of outcomes might be $\mathcal{Y} = [0,200]$ indicating blood sugar levels in mg/dL. For an ad slot on a webpage $x$, the set of interventions $\mathcal{T}$ might be all possible ads in the inventory that fit that slot, while the potential outcomes could be $\mathcal{Y} = \{click, no\_click\}$. For a context $x$ (e.g. patient, webpage), and for each potential intervention $t \in \mathcal{T}$, let $Y_t(x)\in \mathcal{Y}$ be the \emph{potential outcome} for $x$. The fundamental problem of causal inference is that only one potential outcome is observed for a given context $x$: even if we give the patient one medication and later the other, the patient is not in exactly the same state. In machine learning this type of partial feedback is often called ``bandit feedback''. The model described above is known as the Rubin-Neyman causal model \citep{rubin1974estimating,rubin2011causal}.

We are interested in the case of a binary action set $\mathcal{T} = \{0,1\}$, where action $1$ is often known as the ``treated'' and action $0$ is the ``control''. In this case the quantity $Y_1(x) - Y_0(x)$ is of high interest: it is known as the \emph{individualized treatment effect} (ITE) for context $x$ \cite{van2007causal,weiss2015treatment}. Knowing this quantity enables choosing the best of the two actions when confronted with the choice, for example choosing the best treatment for a specific patient.
However, the fact that we only have access to the outcome of one of the two actions prevents the ITE from being known. Another commonly sought after quantity is the \emph{average treatment effect}, $\text{ATE} = \mathbb{E}_{x\sim p(x)}[\text{ITE}(x)]$ for a population with distribution $p(x)$. In the binary action setting, we refer to the observed and unobserved outcomes as the \emph{factual} outcome $y^F(x)$, and \emph{counterfactual} outcome $y^{CF}(x)$ respectively.

A common approach for estimating the ITE is by \emph{direct modelling}: given $n$ samples $\{(x_i,t_i,y^F_i)\}_{i=1}^n$, where $y^F_i = t_i \cdot Y_{1}(x_i) + (1-t_i) Y_{0}(x_i)$, learn a function $h: \mathcal{X} \times \mathcal{T} \rightarrow \mathcal{Y}$ such that $h(x_i,t_i) \approx y^F_i$. The estimated transductive ITE is then:
\begin{equation}
  \label{eq:ite}
\hat{\text{ITE}}(x_i) = \begin{cases}
        y^F_i - h(x_i,1-t_i), & t_i = 1.\\
        h(x_i,1-t_i) - y^F_i, & t_i = 0.
        \end{cases}
\end{equation}

While in principle any function fitting model might be used for estimating the ITE \citep{prentice1976use,gelman2006data,chipman2010bart,wager2015estimation,weiss2015treatment}, it is important to note how this task differs from standard supervised learning. The problem is as follows: the observed sample consists of the set $\hat{P}^F = \{(x_i,t_i)\}_{i=1}^n$. However, calculating the ITE requires inferring the outcome on the set $\hat{P}^{CF} = \{(x_i,1-t_i)\}_{i=1}^n$. We call the set $\hat{P}^F \sim P^F$ the empirical \emph{factual distribution}, and the set $\hat{P}^{CF} \sim P^{CF}$ the empirical \emph{counterfactual distribution}, respectively. Because $P^F$ and $P^{CF}$ need not be equal, the problem of causal inference by counterfactual prediction might require inference over a different distribution than the one from which samples are given. In machine learning terms, this means that the feature distribution of the test set differs from that of the train set. This is a case of \emph{covariate shift}, which is a special case of domain adaptation \citep{daume2006domain,jiang2008literature,mansour2009domain}. A somewhat similar connection was noted in \citet{ScholkopfJPSZMJ2012} with respect to covariate shift, in the context of a very simple causal model.

Specifically, we have that $P^F(x,t) = P(x) \cdot P(t|x)$ and $P^{CF}(x,t) = P(x) \cdot P(\neg t|x)$.  The difference between the observed (factual) sample and the sample we must perform inference on lies precisely in the treatment assignment mechanism, $P(t|x)$. For example, in a randomized control trial, we typically have that $t$ and $x$ are independent. In the contextual bandit setting, there is typically an algorithm which determines the choice of the action $t$ given the context $x$. In observational studies, which are the focus of this work, the treatment assignment mechanism is not under our control and in general will not be independent of the context $x$. Therefore, in general, the counterfactual distribution will be different from the factual distribution.

\begin{figure}[tbp!]
\centering
\includegraphics[width=0.95\columnwidth]{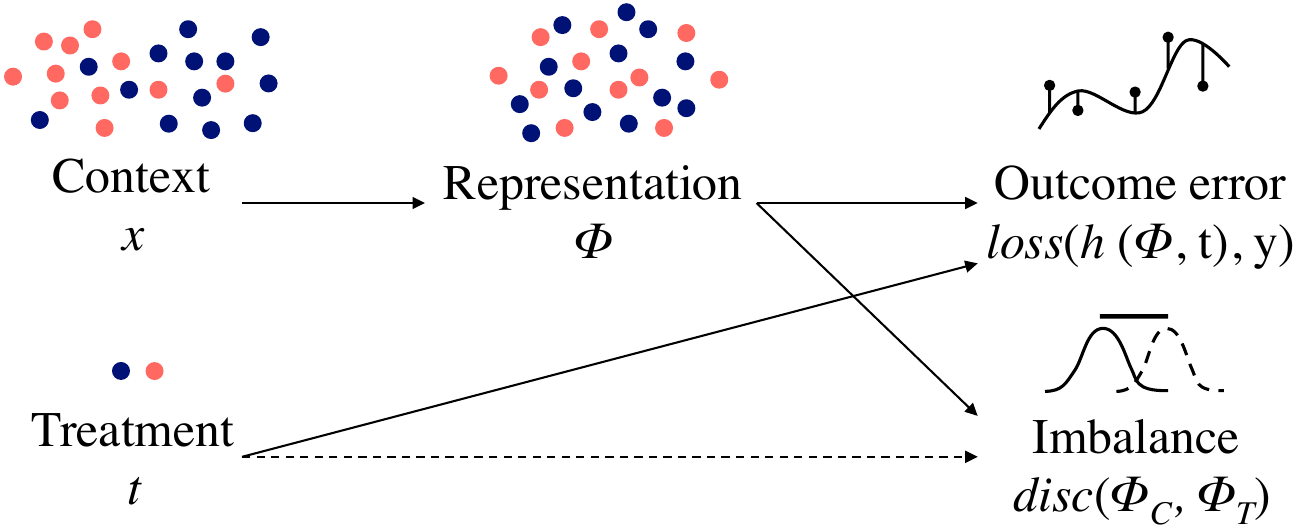}
\caption{\label{fig:model} Contexts $x$ are representated by $\Phi(x)$, which are used, with group indicator $t$, to predict the response $y$ while minimizing the imbalance in distributions measured by $\disc(\Phi_C, \Phi_T)$.}
\end{figure}

\section{Balancing counterfactual regression}\label{Sec:model}
\begin{algorithm}[tbp]
\caption{Balancing counterfactual regression}
\label{alg:model}
\begin{algorithmic}[1]
  \STATE \textbf{Input:} $X, T, Y^F; \hypspace, \repspace; \alpha, \gamma, \lambda$
  \STATE $\displaystyle \Phi^*, g^* = \argmin_{\Phi\in \repspace, g\in \hypspace} B_{\hypspace, \alpha, \gamma}(\Phi, g)$ \;\;\; \eqref{eq:modelbound}
  \vspace{0.1em}
  \STATE $h^* = \argmin_{h \in \cH} \frac{1}{n}\sum_{i=1}^n(h(\Phi,t_i) - y^F_i)^2 + \lambda\|h\|_{\hypspace} $
  \vspace{0.2em}
  \STATE \textbf{Output:} $h^*, \Phi^*$
\end{algorithmic}
\end{algorithm}
\vskip -5pt
We propose to perform counterfactual inference by amending the direct modeling approach, taking into account the fact that the learned estimator $h$ must generalize from the factual distribution to the counterfactual distribution.

Our method, see Figure \ref{fig:model}, learns a representation $\Phi : \mathcal{X} \rightarrow \mathbb{R}^d$, (either using a deep neural network, or by feature re-weighting and selection), and a function $h : \mathbb{R}^d \times \mathcal{T} \rightarrow \mathbb{R}$, such that the learned representation trades off three objectives: (1) enabling low-error prediction of the observed outcomes over the factual representation, (2) enabling low-error prediction of unobserved counterfactuals by taking into account relevant factual outcomes, and (3) the distributions of treatment populations are similar or \emph{balanced}.

We accomplish low-error prediction by the usual means of error minimization over a training set and regularization in order to enable good generalization error. We accomplish the second objective by a penalty  that encourages counterfactual predictions to be close to the nearest observed outcome from the respective treated or control set. Finally, we accomplish the third objective by minimizing the so-called \emph{discrepancy distance}, introduced by \citet{mansour2009domain}, which is a hypothesis class dependent distance measure tailored for domain adaptation. For hypothesis space $\hypspace$, we denote the discrepancy distance by $\disc_\hypspace$. See Section \ref{Sec:theory} for the formal definition and motivation. Other discrepancy measures such as Maximum Mean Discrepancy \citep{gretton2012mmd} could also be used for this purpose. 

Intuitively, representations that reduce the discrepancy between the treated and control populations prevent the learner from using ``unreliable'' aspects of the data when trying to generalize from the factual to counterfactual domains. For example, if in our sample almost no men ever received medication A, inferring how men would react to medication A is highly prone to error and a more conservative use of the gender feature might be warranted.

Let $X = \{x_i\}_{i=1}^n$, $T = \{t_i\}_{i=1}^n$, and $Y^F = \{y_i^F\}_{i=1}^n$ denote the observed units, treatment assignments and factual outcomes respectively. We assume $\mathcal{X}$ is a metric space with a metric $\mathrm{d}$. Let $j(i) \in \argmin_{j \in \{1\ldots n\} \text{ s.t. } t_j = 1-t_i} \mathrm{d}(x_j,x_i)$ be the nearest neighbor of $x_i$ among the group that received the opposite treatment from unit $i$. Note that the nearest neighbor is computed once, in the input space, and does \emph{not} change with the representation $\Phi$. The objective we minimize over representations $\Phi$ and hypotheses $h \in \hypspace$ is
\vspace{-2mm}
\begin{align}
& B_{\hypspace, \alpha, \gamma}(\Phi, h) = \frac{1}{n} \sum_{i=1}^n | h(\Phi(x_i),t_i)- y_i^F |\, + \label{eq:modelbound}\\
& \alpha \disc{}_\hypspace(\pF,\pC) +  \frac{\gamma}{n} \sum_{i=1}^n |h(\Phi(x_i),1-t_i)  - y_{j(i)}^F|~,\nonumber
\vspace{-2mm}
\end{align}
where $\alpha, \gamma > 0$ are hyperparameters to control the strength of the imbalance penalties, and $\disc{}$ is the discrepancy measure defined in \ref{sec:lindisc}.
When the hypothesis class $\hypspace$ is the class of linear functions, the term $\disc{}_\hypspace(\pF,\pC)$ has a closed form brought in \ref{sec:lindisc} below, and $h(\Phi,t_i) = h^\top [\Phi(x_i) \, \, t_i]$. For more complex hypothesis spaces there is in general no exact closed form for $\disc{}_\hypspace(\pF,\pC)$.

Once the representation $\Phi$ is learned, we fit a final hypothesis minimizing a regularized squared loss objective on the factual data. Our algorithm is summarized in Algorithm~\ref{alg:model}. Note that our algorithm involves two minimization procedures.
In Section \ref{Sec:theory} we motivate our method, by showing that our method of learning representations minimizes an upper bound on the regret error over the counterfactual distribution, using results of \citet{cortes2014domain}.

\subsection{Balancing variable selection} \label{sec:varsel}
A na\"{i}ve way of obtaining a balanced representation is to use only features that are already well balanced, i.e. features which have a similar distribution over both treated and control sets. However, imbalanced features can be highly predictive of the outcome, and should not always be discarded. A middle-ground is to restrict the influence of imbalanced features on the predicted outcome. We build on this idea by learning a sparse re-weighting of the features that minimizes the bound in Theorem~\ref{thrm1}. The re-weighting determines the influence of a feature by trading off its predictive capabilities and its balance.

We implement the re-weighting as a diagonal matrix $W$, forming the representation $\Phi(x) = Wx$, with $diag(W)$ subject to a simplex constraint to achieve sparsity. Let $\repspace = \{x \mapsto Wx : W = \diag(w), \; w_i \in [0,1], \; \sum_i w_i = 1\} $ denote the space of such representations. We can now apply Algorithm \ref{alg:model} with $\hypspace_l$ the space of linear hypotheses.
Because the hypotheses are linear, $\disc(\Phi)$ is a function of the distance between the weighted population means, see Section~\ref{sec:lindisc}. With $p = \mathbb{E}[t], c = p - 1/2$, $n_t = \sum_{i=1}^n t_i$, $\mu_1 = \frac{1}{n_t}\sum_{ i: t_i = 1}^n x_i$, and $\mu_0$ analogously defined,
\begin{align*}
\disc{}_{\hypspace_l}(XW) = c + \sqrt{c^2 + \|W(p\mu_1 - (1-p)\mu_0)]\|_2^2}
\end{align*}
To minimize the discrepancy, features $k$ that differ a lot between treatment groups will receive a smaller weight $w_k$. Minimizing the overall objective $B$, involves a trade-off between maximizing balance and predictive accuracy. We minimize~\eqref{eq:modelbound} using alternating sub-gradient descent.

\subsection{Deep neural networks}

Deep neural networks have been shown to successfully learn good representations of high-dimensional data in many tasks~\cite{bengio2013representation}. Here we show that they can be used for counterfactual inference and, crucially, for accommodating imbalance penalties.
We propose a modification of the standard feed-forward architecture with fully connected layers, see Figure~\ref{fig:neuralnet}. The first $d_r$ hidden layers are used to learn a representation $\Phi(x)$ of the input $x$. The output of the $d_r$:th layer is used to calculate the discrepancy  $\disc_\hypspace(\pF,\pC)$. The $d_o$ layers following the first $d_r$ layers take as additional input the treatment assignment $t_i$ and generate a prediction $h([\Phi(x_i), \, t_i])$  of the outcome.

\begin{figure}[t!]
  \centering
  \includegraphics[width=0.95\columnwidth]{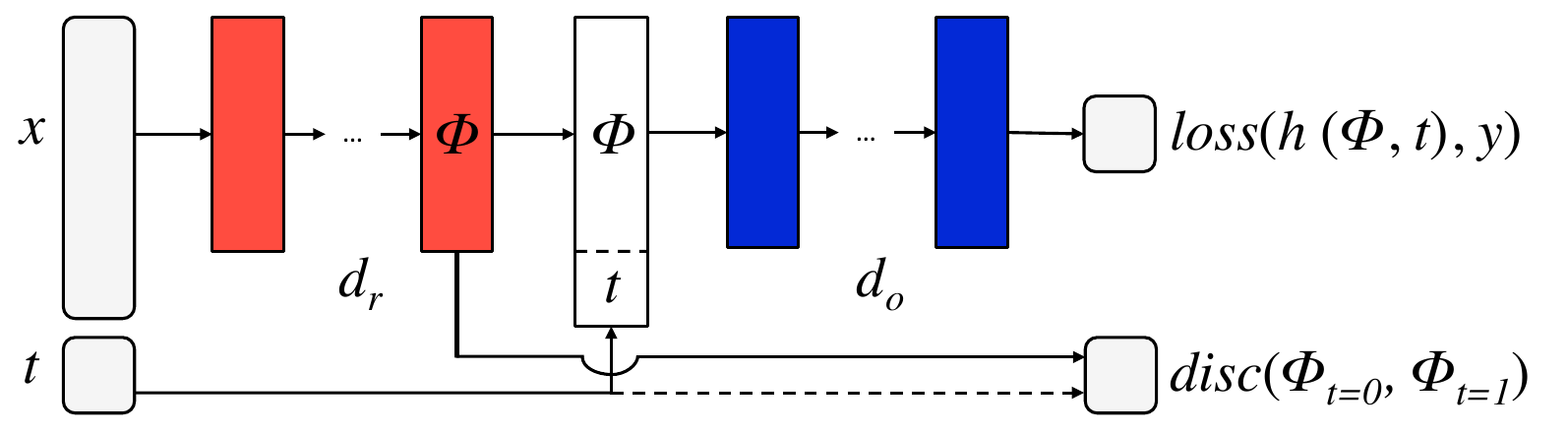}
  \caption{\label{fig:neuralnet}Neural network architecture.}
\end{figure}

\subsection{Non-linear hypotheses and individual effect}
We note that both in the case of variable re-weighting, and for neural nets with a single linear outcome layer,  the hypothesis space $\hypspace$ comprises linear functions of $[\Phi, t]$ and the discrepancy, $\disc_{\hypspace}(\Phi)$ can be expressed in closed-form. A less desirable consequence is that such models cannot capture difference in the individual treatment effect, as they involve no interactions between $\Phi(x)$ and $t$. Such interactions could be introduced by for example (polynomial) feature expansion, or in the case of neural networks, by adding non-linear layers after the concatenation $[\Phi(x), t]$. For both approaches however, we no longer have a closed form expression for $\disc{}_\hypspace(\pF,\pC)$.

\section{Theory}\label{Sec:theory}

In this section we derive an upper bound on the relative counterfactual generalization error of a representation function $\Phi$. The bound only uses quantities we can measure directly from the available data. In the previous section we gave several methods for learning representations which approximately minimize the upper bound.

Recall that for an observed context or instance $x_i \in \mathcal{X}$ with observed treatment $t_i \in \{0,1\}$, the two potential outcomes are $Y_0(x_i), Y_1(x_i) \in \mathcal{Y}$, of which we observe the \emph{factual} outcome $y^F_i = t_i Y_1(x_i) + (1-t_i)Y_0(x_i)$.
Let $ (x_1, t_1, y^F_1), \ldots, (x_n, t_n, y^F_n)$ be a sample from the factual distribution.
Similarly, let $(x_1, 1-t_1, y^{CF}_1), \ldots, (x_n, 1-t_n, y^{CF}_n)$  be the counterfactual sample. Note that while we know the factual outcomes $y^F_i$, we do not know the counterfactual outcomes $y^{CF}_i$.
Let $\Phi : \mathcal{X} \rightarrow \R^d$ be a representation function, and let $\mathcal{R}(\Phi)$ denote its range.
Denote by $\hat{P}_\Phi^F$ the empirical distribution over the representations and treatment assignments $ (\Phi(x_1), t_1),  \ldots, (\Phi(x_n), t_n)$, and similarly $\hat{P}_\Phi^{CF}$ the empirical distribution over the representations and counterfactual treatment assignments $ (\Phi(x_1), 1-t_1),  \ldots, (\Phi(x_n), 1-t_n)$.
Let $\hypspace_l$ be the hypothesis set of linear functions $\beta : \mathcal{R}(\Phi) \times \{0,1\} \rightarrow \mathcal{Y}$.

\begin{definition}[\citealt{mansour2009domain}]
Given a hypothesis set $\hypspace$ and a loss function $L$, the empirical discrepancy between
the empirical distributions $\pF$ and $\pC$ is:
\begin{align*}
&\discPQH = \\
&\max_{\beta, \beta' \in \hypspace} \left| \underset{x \sim \pF}{\mathbb{E}}[L(\beta(x),\beta'(x))] - \underset{x \sim \pC}{\mathbb{E}}[L(\beta(x),\beta'(x))]\right|
\end{align*}
\end{definition}

where $L$ is a loss function $L : \mathcal{Y} \times \mathcal{Y} \rightarrow \R$ with weak Lipschitz constant $\mu$ relative to $\hypspace$ \footnote{When $L$ is the squared loss we can show that if  $\|\Phi(x)\|_2 \leq m$ and $|y|\leq M$, and the hypothesis set $\mathcal{H}$ is that of linear functions with norm bounded by $m/\lambda$, then $\mu \leq 2M (1 + m^2/\lambda )$.}.
Note that the discrepancy is defined with respect to a hypothesis class and a loss function, and is therefore very useful for obtaining generalization bounds involving different distributions. Throughout this section we always have $L$ denote the squared loss.
We prove the following, based on \citet{cortes2014domain}:
\begin{theorem}\label{thrm1}
For a sample $\{(x_i,t_i,y_i^F)\}_{i=1}^n$,  $x_i \in \mathcal{X}$, $t_i \in \{0,1\}$ and $y_i \in \mathcal{Y} $, and a given representation function $\Phi :\mathcal{X} \rightarrow \R^d$, let $\pF = (\Phi(x_1), t_1),  \ldots, (\Phi(x_n), t_n)$, $\pC = (\Phi(x_1), 1-t_1),  \ldots, (\Phi(x_n), 1-t_n) $. We assume that $\mathcal{X}$ is a metric space with metric $\mathrm{d}$, and that the potential outcome functions $Y_0(x)$  and $Y_1(x)$ are Lipschitz continuous with constants $K_0$ and $K_1$ respectively, such that $\mathrm{d}(x_a,x_b) \leq c \implies |Y_t(x_a) - Y_t(x_b)| \leq K_t \cdot c\,$ \, for $t=0,1$.

Let $\hypspace_l \subset\R^{d+1}$ be the space of linear functions $\beta: \mathcal{X} \times \left\{0,1\right\} \rightarrow \mathcal{Y}$, and for $\beta \in \hypspace_l$, let $\mathcal{L}_{P}(\beta) = \mathbb{E}_{(x,t,y) \sim P} \left[L(\beta(x,t),y)\right]$ be the expected loss of $\beta$ over distribution $P$. Let $r = max\left(\mathbb{E}_{(x,t) \sim P^F}\left[\|[\Phi(x), t]\|_2\right],\mathbb{E}_{(x,t) \sim P^{CF}}\left[\|[\Phi(x), t]\|_2\right]\right)$ be the maximum expected radius of the distributions.
For $\lambda >0$, let $\bfp = \argmin_{\beta \in \hypspace_l} \mathcal{L}_{\pF}(\beta) + \lambda \|\beta\|_2^2$, and  $\bcp$ similarly for $\pC$, i.e. $\bfp$ and $\bcp$ are the ridge regression solutions for the factual and counterfactual empirical distributions, respectively.

Let $\hat{y}_i^F(\Phi,h) = h^\top [\Phi(x_i) , t_i]$ and $\hat{y}_i^{CF}(\Phi,h) = h^\top [\Phi(x_i) , \, 1-t_i]$ be the outputs of the hypothesis $h \in \hypspace_l$ over the representation $\Phi(x_i)$ for the factual and counterfactual settings of $t_i$, respectively. Finally, for each $i, j \in \{1 \ldots n\}$, let $\mathrm{d}_{i,j} \equiv \mathrm{d}(x_i,x_j)$ and $j(i) \in \argmin_{j \in \{1\ldots n\} \text{ s.t. } t_j = 1-t_i} \mathrm{d}(x_j,x_i)$ be the nearest neighbor in $\mathcal{X}$ of $x_i$ among the group that received the opposite treatment from unit $i$.
Then for \emph{both} $Q = P^F$ and $Q=P^{CF}$ we have:

\begin{align}
&\frac{\lambda}{\mu r} (\mathcal{L}_{Q}(\bfp)-\mathcal{L}_{Q}(\bcp))^2 \leq \nonumber \\
& \discPQHl \;\; + \label{eq:disc} \\
& \min_{h\in \hypspace_l} \frac{1}{n} \sum_{i=1}^n \left( | \hat{y}_i^F(\Phi,h)- y_i^F |  + |\hat{y}_i^{CF}(\Phi,h)  - y_i^{CF}|\right)\label{eq:eta} \leq \\
& \discPQHl + \nonumber \\
& \min_{h\in \hypspace_l} \frac{1}{n} \sum_{i=1}^n \left( | \hat{y}_i^F(\Phi,h)- y_i^F |  + |\hat{y}_i^{CF}(\Phi,h)  - y_{j(i)}^F|\right) + \label{eq:g}\\
& \frac{K_0}{n} \sum_{i: t_i=1} \mathrm{d}_{i,j(i)} + \frac{K_1}{n} \sum_{i: t_i=0} \mathrm{d}_{i,j(i)} \label{eq:nn} .
\end{align}
\end{theorem}

The proof is in the supplemental material.

Theorem \ref{thrm1} gives, for all fixed representations $\Phi$, a bound on the relative error for a ridge regression model fit on the factual outcomes and evaluated on the counterfactual, as compared with ridge regression had it been fit on the unobserved counterfactual outcomes. It does not take into account how $\Phi$ is obtained, and applies even if $h(\Phi(x),t)$ is not convex in $x$, e.g. if $\Phi$ is a neural net. Since the bound in the theorem is true for all representations $\Phi$, we can attempt to minimize it over $\Phi$, as done in Algorithm \ref{alg:model}.

The term on line \eqref{eq:eta} of the bound includes the unknown counterfactual outcomes $y^{CF}_i$. It measures how well could we in principle fit the factual and counterfactual outcomes together using a linear hypothesis over the representation $\Phi$. For example, if the dimension of the representation is greater than the number of samples, and in addition if there exist constants $b$ and $\epsilon$ such that $|y_i^F - y_i^{CF} - b| \leq \epsilon$, then this term is upper bounded by $\epsilon$. In general however, we cannot directly control its magnitude.

The term on line \eqref{eq:disc} measures the discrepancy between the factual and counterfactual distributions over the representation $\Phi$. In \ref{sec:lindisc} below, we show that this term is closely related to the norm of the difference in means between the representation of the control group and the treated group. A representation for which the means of the treated and control are close (small value of \eqref{eq:disc}), but at the same time allows for a good prediction of the factuals and counterfactuals (small value of \eqref{eq:eta}), is guaranteed to yield structural risk minimizers with similar generalization errors between factual and counterfactual.

We further show that the term on line \eqref{eq:eta}, which cannot be evaluated since we do not know $y^{CF}_i$, can be upper bounded by a sum of the terms on lines \eqref{eq:g} and \eqref{eq:nn}. The term \eqref{eq:g} includes two empirical data fitting terms: $| \hat{y}_i^F(\Phi,v)- y_i^F |$ and  $|\hat{y}_i^{CF}(\Phi,v)  - y_{j(i)}^F|$. The first is simply fitting the observed factual outcomes using a linear function over the representation $\Phi$. The second term is a form of nearest-neighbor regression, where the counterfactual outcomes for a treated (resp. control) instance are fit to the most similar factual outcome among the control (resp. treated) set, where similarity is measured in the original space $\mathcal{X}$.
Finally, the term on line \eqref{eq:nn}, is the only quantity which is independent of the representation $\Phi$. It measures the average distance between each treated instance to the nearest control, and vice-versa, scaled by the Lipschitz constants of the true treated and control outcome functions. This term will be small when: (a) the true outcome functions $Y_0(x)$ and $Y_1(x)$ are relatively smooth, and (b) there is overlap between the treated and control groups, leading to small average nearest neighbor distance across the groups. It is well-known that when there is not much overlap between treated and control, causal inference in general is more difficult since the extrapolation from treated to control and vice-versa is more extreme \citep{rosenbaum2009design}.

The upper bound in Theorem \ref{thrm1} suggests the following approach for counterfactual regression. First minimize the terms \eqref{eq:disc} and \eqref{eq:g} as functions of the representation $\Phi$. Once $\Phi$ is obtained, perform a ridge regression on the factual outcomes using the representations $\Phi(x)$ and the treatment assignments as input. The terms in the bound ensure that $\Phi$ would have a good fit for the data (term \eqref{eq:g}), while removing aspects of the treated and control which create a large discrepancy term \eqref{eq:disc}). For example, if there is a feature which is much more strongly associated with the treatment assignment than with the outcome, it might be advisable to not use it \cite{pearl2011invited}.

\subsection{Linear discrepancy}
\label{sec:lindisc}
A straightforward calculation shows that for a class $\hypspace_l$ of linear hypotheses, $$\disc_{\hypspace_l}(P,Q) = \|\mu_2(P) - \mu_2(Q)\|_2~.$$ Here, $\|A\|_2$ is the spectral norm of $A$ and $\mu_2(P) = \expect_{x \sim P}[xx^\top]$ is the second-order moment of $x \sim P$. In the special case of counterfactual inference, $P$ and $Q$ differ only in the treatment assignment. Specifically,
\begin{align}
\discPQ &=&
\left\|
\begin{bmatrix}
    0_{d,d} & v \\
    v^\top & 2p-1
\end{bmatrix}
\right\|_2 \\
&=& p - \frac{1}{2} + \sqrt{\frac{(2p-1)^2}{4} + \|v\|_2^2}
\label{eq:lindisc}
\end{align}
where $v = \expect_{(x,t) \sim \pF}[\Phi(x)\cdot t] -  \expect_{(x,t) \sim \pF}[\Phi(x) \cdot (1-t)]$ and $p = \expect[t]$.

Let $\mu_1(\Phi) = \expect_{(x,t) \sim \pF}[\Phi(x)|t=1]$ and $\mu_0(\Phi) = \expect_{(x,t) \sim \pF}[\Phi(x)|t=0]$ be the treated and control means in $\Phi$ space. Then $v = p\cdot \mu_1(\Phi) - (1-p)\cdot \mu_0(\Phi)$, exactly the difference in means between the treated and control groups, weighted by their respective sizes.
As a consequence, minimizing the discrepancy with linear hypotheses constitutes matching means in feature space.

\section{Related work}

Counterfactual inference for determining causal effects  in observational studies has been studied extensively in statistics, economics, epidemiology and sociology \citep{morgan2014counterfactuals,robins2000marginal,rubin2011causal,chernozhukov2013inference} as well as in machine learning \citep{langford2011doubly,bottou2013counterfactual,swaminathan2015batch}.

Non-parametric methods do not attempt to model the relation between the context, intervention, and outcome. The methods include nearest-neighbor matching, propensity score matching, and propensity score re-weighting \citep{rosenbaum1983central, rosenbaum2002observational,austin2011introduction}.

Parametric methods, on the other hand, attempt to concretely model the relation between the context, intervention, and outcome. These methods include any type of regression including linear and logistic regression \citep{prentice1976use,gelman2006data}, random forests \citep{wager2015estimation}  and regression trees \citep{chipman2010bart}.

Doubly robust methods combine aspects of parametric and non-parametric methods, typically by using a propensity score weighted regression \citep{bang2005doubly,dudik2011doubly}. They are especially of use when the treatment assignment probability is known, as is the case for off-policy evaluation or learning from logged bandit data. Once the treatment assignment probability has to be estimated, as is the case in most observational studies, their efficacy might wane considerably \citep{kang2007demystifying}.

\citet{tian2014simple} presented one of the few methods that achieve balance by transforming or selecting covariates, modeling  interactions between treatment and covariates.

\section{Experiments}
We evaluate the two variants of our algorithm proposed in Section~\ref{Sec:model} with focus on two questions: 1) What is the effect of imposing imbalance regularization on representations? 2) How do our methods fare against established methods for counterfactual inference? We refer to the variable selection method of Section~\ref{sec:varsel} as \emph{Balancing Linear Regression} ({\sc BLR}) and  the neural network approach as {\sc BNN} for \emph{Balancing Neural Network}.

We report the RMSE of the estimated individual treatment effect, denoted  $\epsilon_{ITE}$, and the absolute error in estimated average treatment effect, denoted $\epsilon_{ATE}$, see Section~\ref{Sec:prob}.
Further, following~\citet{hill2011bayesian}, we report the \emph{Precision in Estimation of Heterogeneous Effect} (PEHE), $\text{PEHE} = \sqrt{\frac{1}{n} \sum_{i=1}^n \left(\hat{y}_1(x_i) - \hat{y}_0(x_i) - (Y_1(x_i)-Y_0(x_i))\right)^2}$. Unlike for ITE, obtaining a good (small) PEHE requires accurate estimation of both the factual and counterfactual responses, not just the counterfactual.
Standard methods for hyperparameter selection, including cross-validation, are unavailable when training counterfactual models on real-world data, as there are no samples from the counterfactual outcome. In our experiments, all outcomes are simulated, and we have access to counterfactual samples. To avoid fitting parameters to the test set, we generate multiple repeated experiments, each with a different outcome function and pick hyperparameters once, for all models (and baselines), based on a held-out set of experiments. While not possible for real-world data, this approach gives an indication of the robustness of the parameters.

The neural network architectures used for all experiments consist of fully-connected ReLU layers trained using RMSProp, with a small $l_2$ weight decay, $\lambda=10^{-3}$. We evaluate two architectures. {\sc BNN-4-0} consists of 4 ReLU representation-only layers and a single linear output layer, $d_r = 4, d_o = 0$. {\sc BNN-2-2} consists of 2 ReLU representation-only layers, 2 ReLU output layers after the treatment has been added, and a single linear output layer, $d_r = 2, d_o=2$, see Figure~\ref{fig:neuralnet}. For the IHDP data we use layers of 25 hidden units each. For the News data representation layers have 400 units and output layers 200 units. The nearest neighbor term, see Section~\ref{Sec:model}, did not improve empirical performance, and was omitted for the BNN models. For the neural network models, the hypothesis and the representation were fit jointly.

We include several different linear models in our comparison, including ordinary linear regression ({\sc OLS}) and doubly robust linear regression (DR)~\cite{bang2005doubly}. We also include a method were variables are first selected using LASSO and then used to fit a ridge regression ({\sc Lasso + Ridge}). Regularization parameters are picked based on a held out sample. For DR, we estimate propensity scores using logistic regression and clip weights at 100. For the News dataset (see below), we perform the logistic regression on the first 100 principal components of the data.

Bayesian Additive Regression Trees (BART)~\cite{chipman2010bart} is a non-linear regression model which has been used successfully for counterfactual inference in the past~\cite{hill2011bayesian}. We compare our results to BART using the implementation provided in the BayesTree R-package~\cite{BayesTree}. Like~\cite{hill2011bayesian}, we do not attempt to tune the parameters, but use the default. Finally, we include a standard feed-forward neural network, trained with 4 hidden layers, to predict the factual outcome based on $X$ and $t$, without a penalty for imbalance. We refer to this as {\sc NN-4}.

\subsection{Simulation based on real data -- IHDP}
\citet{hill2011bayesian} introduced a semi-simulated dataset based on the Infant Health and Development Program (IHDP). The IHDP data has covariates from a real randomized experiment, studying the effect of high-quality child care and home visits on future cognitive test scores. The experiment proposed by \citet{hill2011bayesian} uses a simulated outcome and artificially introduces imbalance between treated and control subjects by removing a subset of the treated population. In total, the dataset consists of 747 subjects (139 treated, 608 control), each represented by 25 covariates measuring properties of the child and their mother. For details, see \citet{hill2011bayesian}. We run 100 repeated experiments for hyperparameter selection and 1000 for evaluation, all with the log-linear response surface implemented as setting ``A" in the NPCI package~\cite{npci}.

\subsection{Simulation based on real data -- News }
We introduce a new dataset, simulating the opinions of a media consumer exposed to multiple news items. Each item is consumed either on a mobile device or on desktop. The units are different news items represented by word counts $x_i \in \mathbb{N}^V$, and the outcome $y^F(x_i) \in \mathbb{R}$ is the readers experience of $x_i$. The intervention $t \in \{0, 1\}$ represents the viewing device, desktop $(t=0)$ or mobile $(t=1)$.
We assume that the consumer prefers to read about certain topics on mobile. To model this, we train a topic model on a large set of documents and let $z(x) \in \mathbb{R}^k$ represent the topic distribution of news item $x$. We define two centroids in topic space, $z^c_1$ (mobile), and $z^c_0$ (desktop), and let the readers opinion of news item $x$ on device $t$ be determined by the similarity between $z(x)$ and $z^c_t$,
$
y^F(x_i) = C \left(z(x)^\top z^c_0 + t_i\cdot z(x)^\top z^c_1\right) + \epsilon ~
$,
where $C$ is a scaling factor and $\epsilon \sim \mathcal{N}(0,1)$. Here, we let the mobile centroid, $z^c_1$ be the topic distribution of a randomly sampled document, and $z^c_0$ be the average topic representation of all documents. We further assume that the assignment of a news item $x$ to a device $t \in \{0,1\}$ is biased towards the device preferred for that item. We model this using the softmax function,
$
p(t = 1 \mid x) = \frac{e^{\kappa\cdot z(x)^\top z^c_1}}{e^{\kappa\cdot z(x)^\top z^c_0} + e^{\kappa\cdot z(x)^\top z^c_1}}
$,
where $\kappa \geq 0$ determines the strength of the bias. Note that $\kappa = 0$ implies a completely random device assignment.

We sample $n=5000$ news items and outcomes according to this model, based on 50 LDA topics, trained on documents from the NY Times corpus (downloaded from UCI~\cite{bagofwords}). The data available to the algorithms are the raw word counts, from a vocabulary of $k=3477$ words, selected as union of the most $100$ probable words in each topic. We set the scaling parameters to $C=50, \kappa=10$ and sample $50$ realizations for evaluation. Figure~\ref{fig:topic_model_data} shows a visualization of the outcome and device assignments for a  sample of 500 documents. Note that the device assignment becomes increasingly random, and the outcome lower, further away from the centroids.

\begin{figure}[t!]
  \centering
  \begin{subfigure}[b]{0.49\columnwidth}
    \centering
    \includegraphics[width=1.0\textwidth]{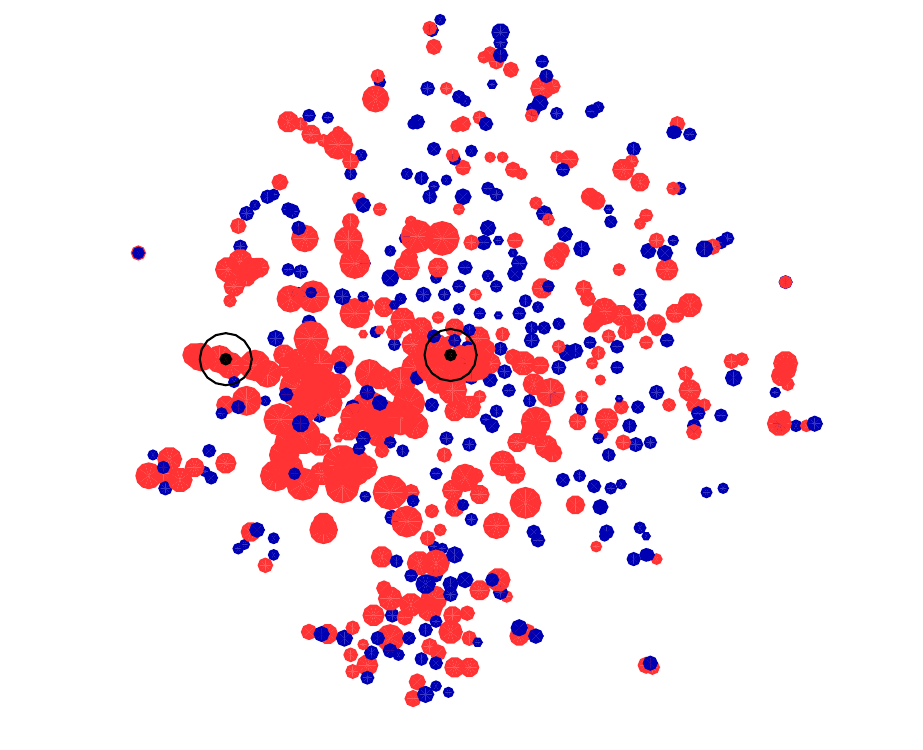}
  \end{subfigure}
  \begin{subfigure}[b]{0.49\columnwidth}
    \centering
    \includegraphics[width=1.0\textwidth]{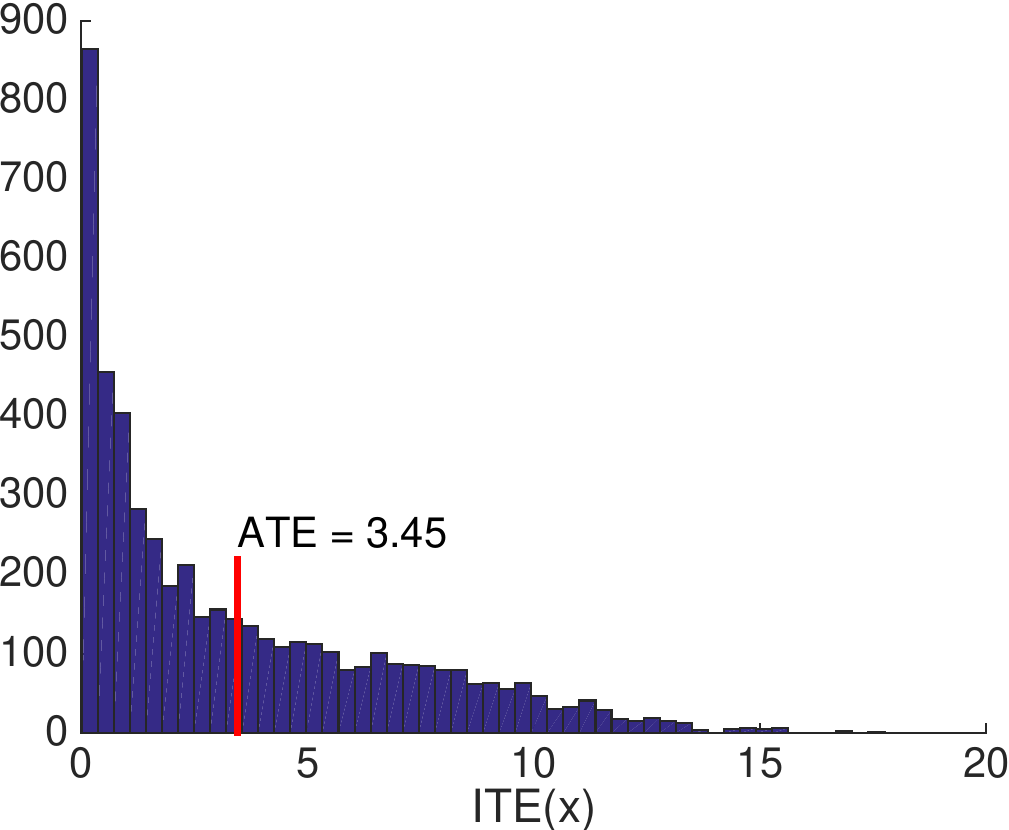}
  \end{subfigure}
  \caption{\label{fig:topic_model_data}Visualization of one of the News sets (left). Each dot represents a single news item $x$. The radius represents the outcome $y(x)$, and the color the treatment $t$. The two black dots represent the two centroids. Histogram of ITE in News (right).}
\end{figure}

\begin{table}[t!]
  \caption{IHDP. Results and standard errors for 1000 repeated experiments. (Lower is better.)
  Proposed methods: {\sc BLR, BNN-4-0} and {\sc BNN-2-2}. $\dagger$~\cite{chipman2010bart}}
  \label{tbl:ihdp_results}
  \vskip 0.15in
  \begin{center}
    \begin{small}
      \begin{sc}
      \begin{tabular}{lccc}
        \hline
        \abovespace\belowspace
        & $\epsilon_{ITE}$ & $\epsilon_{ATE}$ & PEHE \\
        \hline
        \multicolumn{4}{l}{Linear outcome} \\
        OLS       & $4.6\pm 0.2$ & $0.7\pm 0.0$ & $5.8\pm 0.3$  \\
        Doubly Robust   & $3.0\pm 0.1$ & $0.2\pm 0.0$ & $5.7\pm 0.3$  \\
        Lasso + Ridge  & $2.8\pm 0.1$ & $0.2\pm 0.0$ & $5.7\pm 0.2$ \\
        BLR      & $2.8\pm 0.1$ & $0.2\pm 0.0$ & $5.7\pm 0.3$   \\
        BNN-4-0  & $3.0\pm 0.0$ & $0.3\pm 0.0$ & $5.6\pm 0.3$   \\
        \hline
        \multicolumn{4}{l}{Non-linear outcome} \\
        NN-4   & $2.0\pm 0.0$ & $0.5\pm 0.0$ & $1.9\pm 0.1$   \\
        BART$^\dagger$   & $2.1\pm 0.2$ & $0.2\pm 0.0$ & $1.7\pm 0.2$  \\
        BNN-2-2   & $\mathbf{1.7\pm 0.0}$ & $0.3\pm 0.0 $ & $\mathbf{1.6\pm 0.1}$  \\

        \hline
      \end{tabular}
    \end{sc}
    \end{small}
  \end{center}
\end{table}

\begin{table}[t!]
  \caption{News. Results and standard errors for 50 repeated experiments. (Lower is better.)
  Proposed methods: {\sc BLR, BNN-4-0} and {\sc BNN-2-2}. $\dagger$~\cite{chipman2010bart} }
  \label{tbl:topic_results}
  \vskip 0.15in
  \begin{center}
    \begin{small}
      \begin{sc}
      \begin{tabular}{lccc}
        \hline
        \abovespace\belowspace
        & $\epsilon_{ITE}$ & $\epsilon_{ATE}$ & PEHE \\
        \hline
        \multicolumn{4}{l}{Linear outcome} \\
        OLS      & $3.1\pm 0.2$ & $0.2\pm 0.0$ & $3.3\pm 0.2$   \\
        Doubly Robust   & $3.1\pm 0.2$ & $0.2\pm 0.0$ & $3.3\pm 0.2$   \\
        Lasso + Ridge   & $2.2\pm 0.1$ & $0.6\pm 0.0$ & $3.4\pm 0.2$   \\
        BLR      & $2.2\pm 0.1$ & $0.6\pm 0.0$ & $3.3\pm 0.2$   \\
        BNN-4-0  & $2.1\pm 0.0$ & $0.3\pm 0.0$ & $3.4\pm 0.2$   \\
        \hline
        \multicolumn{4}{l}{Non-linear outcome} \\
        NN-4      & $2.8\pm 0.0$ & $1.1\pm 0.0$ & $3.8\pm 0.2$   \\
        BART$^\dagger$   & $5.8\pm 0.2$ & $0.2\pm 0.0$ & $3.2\pm 0.2$   \\
        BNN-2-2   & $\mathbf{2.0\pm 0.0}$ & $0.3\pm 0.0$ & $\mathbf{2.0\pm 0.1}$   \\

        \hline
      \end{tabular}
    \end{sc}
    \end{small}
  \end{center}
\end{table}

\begin{figure}[t]
  \centering
  \begin{subfigure}[b]{0.49\columnwidth}
    \centering
    \includegraphics[width=1.0\textwidth]{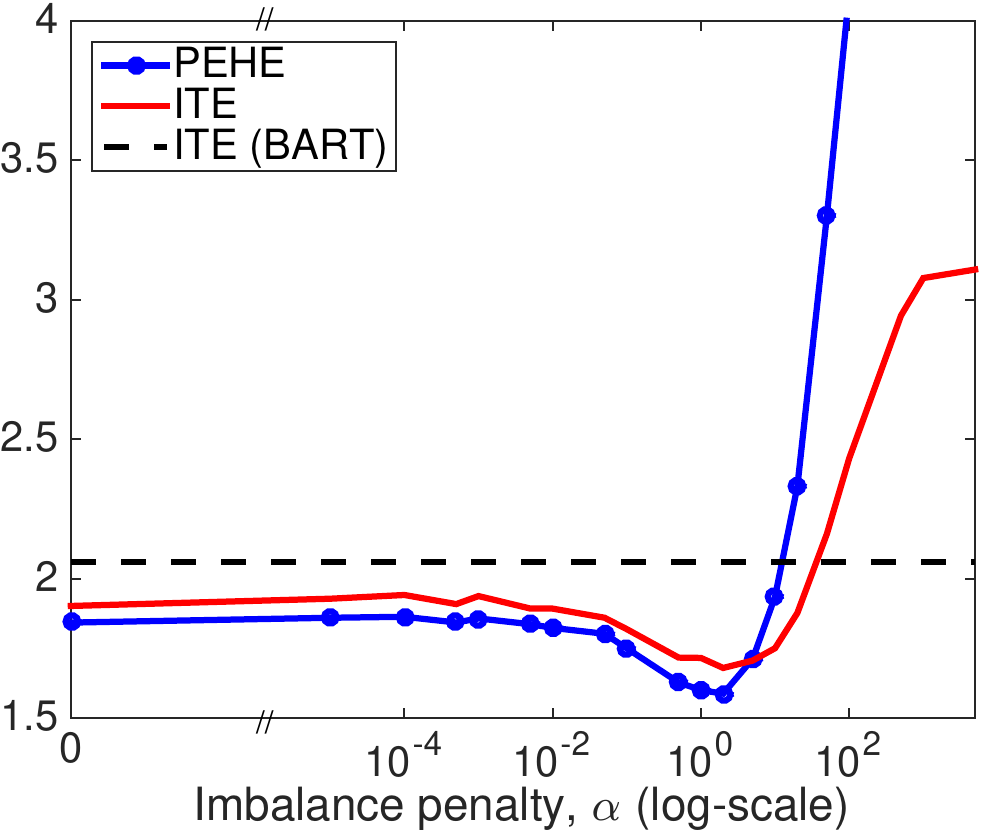}
  \end{subfigure}
  \begin{subfigure}[b]{0.49\columnwidth}
    \centering
    \includegraphics[width=1.0\textwidth]{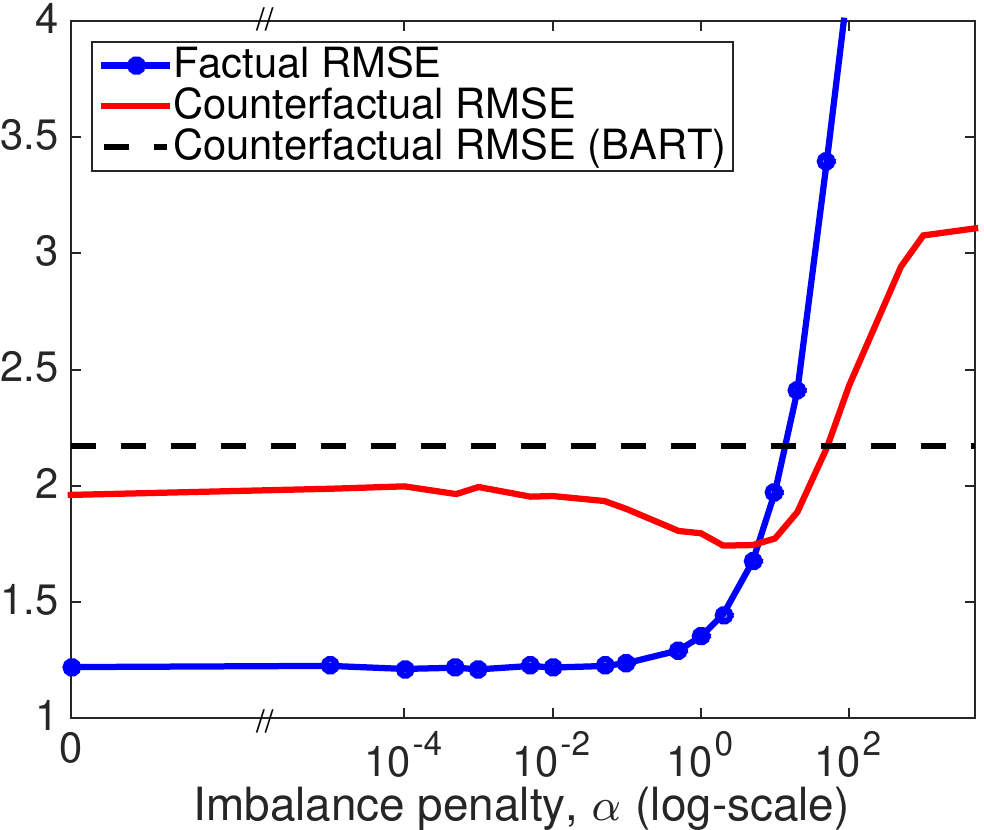}
  \end{subfigure}
  \caption{\label{fig:ihdp_alpha_sweep}Error in estimated treatment effect (ITE, PEHE) and counterfactual response (RMSE) on the IHDP dataset. Sweep over $\alpha$ for the BNN-2-2 neural network model.}
\end{figure}

\subsection{Results}
The results of the IHDP and News experiments are presented in Table~\ref{tbl:ihdp_results} and Table~\ref{tbl:topic_results} respectively. We see that, in general, the non-linear methods perform better in terms of individual prediction (ITE, PEHE). Further, we see that our proposed balancing neural network {\sc BNN-2-2} performs the best on both datasets in terms of estimating the ITE and PEHE, and is competitive on average treatment effect, ATE. Particularly noteworthy is the comparison with the network without balance penalty, {\sc NN-4}. These results indicate that our proposed regularization can help avoid overfitting the representation to the factual outcome. Figure~\ref{fig:ihdp_alpha_sweep} plots the performance of {\sc BNN-2-2} for various imbalance penalties $\alpha$. The valley in the region $\alpha=1$, and the fact that we don't experience a loss in performance for smaller values of $\alpha$, show that the penalizing imbalance in the representation $\Phi$ has the desired effect.

For the linear methods, we see that the two variable selection approaches, our proposed BLR method and {\sc Lasso + Ridge}, work the best in terms of estimating ITE. We would like to emphasize that \lridge{} is a very strong baseline and it's exciting that our theory-guided method is competitive with this approach.

On News, BLR and \lridge{} perform equally well yet again, although this time with qualitatively different results, as they do not select the same variables. Interestingly, BNN-4-0, BLR and \lridge{} all perform better on News than the standard neural network, NN-4. The performance of BART on News is likely hurt by the dimensionality of the dataset, and could improve with hyperparameter tuning.

\section{Conclusion}
As machine learning is becoming a major tool for researchers and policy makers across different fields such as healthcare and economics, causal inference becomes a crucial issue for the practice of machine learning. In this paper we focus on counterfactual inference, which is a widely applicable special case of causal inference. We cast counterfactual inference as a type of domain adaptation problem, and derive a novel way of learning representations suited for this problem.

Our models rely on a novel type of regularization criteria: learning \emph{balanced representations}, representations which have similar distributions among the treated and untreated populations. We show that trading off a balancing criterion with standard data fitting and regularization terms is both practically and theoretically prudent.

Open questions which remain are how to generalize this method for cases where more than one treatment is in question, deriving better optimization algorithms and using richer discrepancy measures.

\section*{Acknowledgements}
DS and US were supported by NSF CAREER award \#1350965.

{\normalsize{
\bibliography{cfr}
\bibliographystyle{icml2016}
}}

\newpage
\appendix
\section{Proof of Theorem 1}\label{app:proof}
We use a result implicit in the proof of Theorem 2 of \citet{cortes2014domain}, for the case where $\hypspace$ is the set of linear hypotheses over a fixed representation $\Phi$. \citet{cortes2014domain} state their result for the case of domain adaptation: in our case, the factual distribution is the so-called ``source domain'', and the counterfactual distribution is the ``target domain''.
\begin{customthm}{A1}\label{thrmA1}[\citet{cortes2014domain}]
Using the notation and assumptions of Theorem 1, for \emph{both} $Q = P^F$ and $Q=P^{CF}$:
\begin{align}\label{eq:discineq}
&\frac{\lambda}{\mu r} (\mathcal{L}_{Q}(\bfp)-\mathcal{L}_{Q}(\bcp))^2 \leq \nonumber \\
& \discPQHl  + \nonumber \\
& \min_{h\in \hypspace_l} \frac{1}{n} \left( \sum_{i=1}^n | \hat{y}_i^F(\Phi,h)- y_i^F |  + |\hat{y}_i^{CF}(\Phi,h)  - y_i^{CF}|\right)
\end{align}
\end{customthm}

In their work, \citet{cortes2014domain} assume the $\hypspace$ is a reproducing kernel Hilbert space (RKHS) for a universal kernel, and they do not consider the role of the representation $\Phi$. Since the RKHS hypothesis space they use is much stronger than the linear space $\hypspace_l$, it is often reasonable to assume that the second term in the bound \ref{eq:discineq} is small. We however cannot make this assumption, and therefore we wish to explicitly bound the term $\min_{h\in \hypspace_l} \frac{1}{n} \left( \sum_{i=1}^n | \hat{y}_i^F(\Phi,h)- y_i^F |  + |\hat{y}_i^{CF}(\Phi,h)  - y_i^{CF}|\right)$, while using the fact that we have control over the representation $\Phi$.

\begin{lemma}\label{lemma1}
Let $\{(x_i,t_i,y_i^F)\}_{i=1}^n$,  $x_i \in \mathcal{X}$, $t_i \in \{0,1\}$ and $y_i^F \in \mathcal{Y} \subseteq \R $. We assume that $\mathcal{X}$ is a metric space with metric $\mathrm{d}$, and that there exist two function $Y_0(x)$ and $Y_1(x)$ such that $y_i^F = t_i Y_1(x_i) + (1-t_i)Y_0(x_i)$, and in addition we define $y_i^{CF} = (1-t_i) Y_1(x_i) + t_i Y_0(x_i)$. We further assume that the functions $Y_0(x)$  and $Y_1(x)$ are Lipschitz continuous with constants $K_0$ and $K_1$ respectively, such that $\mathrm{d}(x_a,x_b) \leq c \implies |Y_t(x_a) - Y_t(x_b)| \leq K_t c$. Define  $j(i) \in \argmin_{j \in \{1\ldots n\} \text{ s.t. } t_j = 1-t_i} \mathrm{d}(x_j,x_i)$ to be the nearest neighbor of $x_i$ among the group that received the opposite treatment from unit $i$, for all $i \in \{1 \ldots n\}$. Let $\mathrm{d}_{i,j} = \mathrm{d}(x_i,x_j)$

For any $b \in \mathcal{Y}$ and $h \in \hypspace$:
\begin{align*}
|b-y_i^{CF} | \leq | b - y_{j(i)}^F| + K_{1-t_i} \mathrm{d}_{i,j(i)}
\end{align*}
\end{lemma}
\begin{proof}
By the triangle inequality, we have that:
$$|b-y_i^{CF} | \leq | b - y_{j(i)}^F|  + |y_{j(i)}^F - y_i^{CF}|.$$
By the Lipschitz assumption on $Y_{1-t_i}$, and since $\mathrm{d}(x_i,x_{j(i)}) \leq \mathrm{d}_{i,j(i)}$, we obtain that $$ |y_{j(i)}^F - y_i^{CF}| = |Y_{1-t_i}(x_{j(i)}) - Y_{1-t_i}(x_i) | \leq \mathrm{d}_{i,j(i)} K_{1-t_i}.$$
By definition $y_i^{CF} = Y_{1-t_i}(x_i)$. In addition, by definition of $j(i)$, we have $t_{j(i)} = 1-t_i$, and therefore $y_{j(i)}^F = Y_{1-t_i}(x_{j(i)})$, proving the equality. The inequality is an immediate consequence of the Lipschitz property.
\end{proof}

We restate Theorem 1 and prove it.
\begin{apptheorem}
For a sample $\{(x_i,t_i,y_i^F)\}_{i=1}^n$,  $x_i \in \mathcal{X}$, $t_i \in \{0,1\}$ and $y_i \in \mathcal{Y} $,  recall that $y_i^F = t_i Y_1(x_i) + (1-t_i)Y_0(x_i)$, and in addition define $y_i^{CF} = (1-t_i) Y_1(x_i) + t_i Y_0(x_i)$. For a given representation function $\Phi :\mathcal{X} \rightarrow \R^d$, let $\pF = (\Phi(x_1), t_1),  \ldots, (\Phi(x_n), t_n)$, $\pC = (\Phi(x_1), 1-t_1),  \ldots, (\Phi(x_n), 1-t_n) $. We assume that $\mathcal{X}$ is a metric space with metric $\mathrm{d}$, and that the potential outcome functions $Y_0(x)$  and $Y_1(x)$ are Lipschitz continuous with constants $K_0$ and $K_1$ respectively, such that $\mathrm{d}(x_a,x_b) \leq c \implies |Y_t(x_a) - Y_t(x_b)| \leq K_t c$.

Let $\hypspace_l \subset\R^{d+1}$ be the space of linear functions, and
for $\beta \in \hypspace_l$, let $\mathcal{L}_{P}(\beta) = \mathbb{E}_{(x,t,y) \sim P} \left[L(\beta(x,t),y)\right]$ be the expected loss of $\beta$ over distribution $P$. Let $r = max\left(\mathbb{E}_{(x,t) \sim P^F}\left[\|[\Phi(x), t]\|_2\right],\mathbb{E}_{(x,t) \sim P^{CF}}\left[\|[\Phi(x), t]\|_2\right]\right)$.
For $\lambda >0$, let $\bfp = \argmin_{\beta \in \hypspace_l} \mathcal{L}_{\pF}(\beta) + \lambda \|\beta\|_2^2$, and  $\bcp$ similarly for $\pC$, i.e. $\bfp$ and $\bcp$ are the ridge regression solutions for the factual and counterfactual empirical distributions, respectively.

Let $\hat{y}_i^F(\Phi,h) = h^\top [\Phi(x_i) , t_i]$ and $\hat{y}_i^{CF}(\Phi,h) = h^\top [\Phi(x_i) , \, 1-t_i]$ be the outputs of the hypothesis $h \in \hypspace_l$ over the representation $\Phi(x_i)$ for the factual and counterfactual settings of $t_i$, respectively. Finally, for each $i \in \{1 \ldots n\}$, let $j(i) \in \argmin_{j \in \{1\ldots n\} \text{ s.t. } t_j = 1-t_i} \mathrm{d}(x_j,x_i)$ be the nearest neighbor of $x_i$ among the group that received the opposite treatment from unit $i$. Let $\mathrm{d}_{i,j} = \mathrm{d}(x_i,x_j)$.

Then for \emph{both} $Q = P^F$ and $Q=P^{CF}$ we have:
\begin{align}
&\frac{\lambda}{\mu r} (\mathcal{L}_{Q}(\bfp)-\mathcal{L}_{Q}(\bcp))^2 \leq \label{ineq1} \\
& \discPQHl  + \nonumber\\
& \min_{h\in \hypspace_l} \frac{1}{n} \sum_{i=1}^n \left( | \hat{y}_i^F(\Phi,h)- y_i^F |  + |\hat{y}_i^{CF}(\Phi,h)  - y_i^{CF}|\right)  \leq  \label{ineq2} \\
& \discPQHl +  \nonumber \\
& \min_{h\in \hypspace_l} \frac{1}{n} \sum_{i=1}^n \left( | \hat{y}_i^F(\Phi,h)- y_i^F |  + |\hat{y}_i^{CF}(\Phi,h)  - y_{j(i)}^F|\right) + \nonumber\\
& \frac{K_0}{n} \sum_{i: t_i=1} \mathrm{d}_{i,j(i)} + \frac{K_1}{n} \sum_{i: t_i=0} \mathrm{d}_{i,j(i)}. \nonumber
\end{align}
\end{apptheorem}
\begin{proof}
Inequality \eqref{ineq1} is immediate by Theorem \ref{thrmA1}. In order to prove inequality \eqref{ineq2}, we apply Lemma \ref{lemma1}, setting $b = \hat{y}_i^{CF}$ and summing over the $i$.
\end{proof}

\end{document}